%% file: paper.tex
\documentclass{article}

\usepackage{microtype}
\usepackage{graphicx}
\usepackage{graphicx,lipsum}
\usepackage{booktabs}
\usepackage{multirow}
\usepackage{colortbl}
\usepackage[dvipsnames]{xcolor}
\usepackage{subcaption}

\usepackage{amsmath}
\usepackage{amssymb}
\usepackage{mathtools}
\usepackage{amsthm}
\usepackage{arxiv}
\usepackage[capitalize,noabbrev]{cleveref}

\crefname{table}{Table}{Tables}
\crefname{section}{Section}{Sections}
\crefname{figure}{Figure}{Figures}
\crefname{appendix}{Appendix}{Appendices}

\theoremstyle{plain}
\newtheorem{theorem}{Theorem}[section]

\newtheorem{lemma}[theorem]{Lemma}

\theoremstyle{definition}
\newtheorem{definition}[theorem]{Definition}

\theoremstyle{remark}
\newtheorem{remark}[theorem]{Remark}

\usepackage[textsize=tiny]{todonotes}

\title{Dynamic Memory Based Adaptive Optimization}
\author{
 \parbox{\linewidth}{\centering
Balázs Szegedy$^{\dag}$, Domonkos Czifra$^{\dag}$, Péter Kőrösi-Szabó$^{\dag}$
}
\\
\\
$^{\dag}$Alfréd Rényi Insititute of Mathematics, Budapest, Hungary\\
\small{\texttt{\{szegedy, doma945, koszpe\}@renyi.hu}}
}
\date{}

\begin{document}
\maketitle

\input{main}

\newpage

\bibliography{paper}
\bibliographystyle{plainnat}

\newpage
\appendix
\onecolumn

\input{appendix}

\end{document}

%% file: main.tex
\begin{abstract}
Define an optimizer as having memory $k$ if it stores $k$ dynamically changing vectors in the parameter space. Classical SGD has memory $0$, momentum SGD optimizer has $1$ and Adam optimizer has $2$. We address the following questions: \textit{How can optimizers make use of more memory units? What information should be stored in them? How to use them for the learning steps?} As an approach to the last question, we introduce a general method called "Retrospective Learning Law Correction" or shortly RLLC. This method is designed to calculate a dynamically varying linear combination (called {\it learning law}) of memory units, which themselves may evolve arbitrarily. We demonstrate RLLC on optimizers whose memory units have linear update rules and small memory ($\leq 4$ memory units). Our experiments show that in a variety of standard problems, these optimizers outperform the above mentioned three classical optimizers. We conclude that RLLC is a promising framework for boosting the performance of known optimizers by adding more memory units and by making them more adaptive. 
\end{abstract}

\section{Introduction}
In this paper, we investigate optimizers that store $k$ vectors in the parameter space $\mathbb{R}^n$ of a neural network or more generally in the parameter space related to any optimization problem. We call such vectors {\it memory units} and we measure the memory usage of an optimizer by the number them. 

The simplest example for an optimizer with memory is the momentum SGD optimizer which stores a single vector $m$ (\textit{momentum vector}) in the parameter space $\mathbb{R}^n$. In each step, $m$ is updated according to the rule $m\longleftarrow \beta m+\bigtriangledown_\theta f(\theta)$ where $f$ is the objective function, $\theta\in\mathbb{R}^n$ is the parameter vector and $0<\beta<1$ is a fixed real number. The vector $\theta$ is updated according to the rule $\theta\longleftarrow\theta-cm$ where $c>0$ is the learning rate. 

The Adam optimizer operates with two memory units. One of them is the \textit{momentum vector} and the other one is the momentum of the squares of the gradient vectors. In contrast with the momentum optimizer, the Adam optimizer is not linear in the gradient vectors. Neither the update rule of the memory units, nor the way the memory units are used for the parameter update is linear. 

The present paper has two independent contributions. The first contribution is a novel and simple method that we call RLLC={\it Retrospective Learning Law Correction}. It is an update rule for a vector $L$ (called {\it learning law}) that describes a natural way of using a set of dynamically changing memory units for the update of the parameter vector $\theta$. More precisely, $L\in\mathbb{R}^k$ contains the coefficients of a linear combination of the $k$ memory units which is multiplied by a fixed learning rate $c_1$ and substracted from $\theta$ as usual. In each step, before updating $\theta$ and the memory units, we update $L$ by the formula $L\longleftarrow L+c_2M^+g$ where $M$ is the $n\times k$ matrix formed by the memory units, $M^+$ is the Moore-Penrose inverse of $M$, $g$ is the newly received gradient and $c_2$ is the {\it meta learning rate}. The main idea behind the update rule for $L$ is that the new gradient $g$ contains retrospective information on how the algorithm could have performed better in the previous step. Thus it can be used to compute a corrected version of $L$ which "thinks more ahead". Note that our update rule of the learning law can also be regarded as a general framework for associating a $k$-dimensional adaptive learning rate with an arbitrary set of $k$ evolving memory units.   

As the second main contribution, we examine optimizers in which memory units are updated by fixed linear rules. More precisely, in each step each memory unit is updated to a linear combination of the memory units and the new arriving gradient. The parameter vector is updated by a (possibly changing) liner combination (given by the learning law $L$) of the memory units. Such optimizers are interesting even if the learning law is fixed. They include SGD, momentum SGD and Nesterov Accelerated Gradient (NAG) \cite{Nesterov2012GradientMF}. Thus, the linear framework provides a useful generalization of these famous optimizers and enables a dynamically changing continuous interpolation between them. The RLLC method turns out to be ideal for this. Our experiments show that linear memory combined with RLLC leads to powerful optimizers. The case of memory $1$ is already interesting. A memory $1$ linear optimizer stores a momentum vector. Applying RLLC in this trivial setting yields a variant of the momentum SGD optimizer enhanced with a new type of adaptive learning rate. As the number of memory units increases, the mathematics becomes more complex, presenting a field of study that is interesting in its own right. We present some of the fundamental properties of linearly updated memory units. In particular, we prove a version of basis independence for RLLC combined with linear memory which allows us to apply basis transformations to the update rules without changing the optimization process. This together with a variant of the Jordan normal form over the field $\mathbb{R}$ helps to convert these optimizers into a canonical form in which each memory unit is associated with a so-called Jordan block. A Jordan block of size $1$ corresponds to a single memory unit (denoted by $M(\beta)$) storing a momentum vector of the gradients with parameter $\beta$. A Jordan block of size $2$ either corresponds to a pair of memory units (denoted by $CM(\beta),\beta\in\mathbb{C}$) namely the real and imaginary parts of a momentum vector with complex parameter or to a pair of memory units $m_1,m_2$ (denoted by $M_2(\beta)$) where $m_1$ is a momentum vector of the gradient and $m_2$ is a momentum vector of $m_1$, both with parameter $\beta$. In general, there are two infinite families of Jordan blocks giving rise to $k$-tuples or $2k$-tuples of memory units denoted by $M_k(\beta)$ and $CM_k(\gamma)$. These are the fundamental building blocks of linearly updated memory. We denote the natural operation by $\oplus$ which combines these building blocks into larger memory by the union of the corresponding memory units. By slightly abusing the notation we often identify memory update rules with optimizers where learning is given by the RLLC method. For example, $M(\beta)$ also denotes the memory $1$ optimizer with memory unit $M(\beta)$ (a momentum vector) and with RLLC. Notice that the $M(\beta)$ optimizer is a close relative of momentum SGD but it is not equivalent with it.

In our experiments, we identified a number of interesting simple settings involving few (at most $4)$ memory units. These include the types of optimizers $M(\beta)$, $M(\beta)\oplus M(0)$, 
$M(\beta_1)\oplus M(\beta_2)\oplus M(\beta_3)$, $M_2(\beta)$, $M_3(\beta)$ and $M(\beta){\oplus}M(-\beta){\oplus}CM(\beta i)$.
We observed that these optimizers often surpassed the performance of three commonly used optimizers across a variety of tasks even without carefully optimizing the parameters $\beta_i$.
Notice that $M(\beta)\oplus M(0)$ is an adaptively changing linear combination of SGD, momentum SGD and NAG. Thus, it adaptively interpolates between three well known optimizers. 
Remarkably, it demonstrated competitive or even superior performance compared to the Adam optimizer in many tasks, which also uses two memory units.

This paper primarily aims not to challenge all existing optimizers in the field, but rather to introduce a novel mathematical concept that could spark further research. The experimental results presented here should be interpreted as an illustration of the potential of our approach. We posit that the implications of the RLLC method extend beyond mere enhancements to current optimization techniques, suggesting broader applications and insights in the realm of optimization and machine learning.

\bigskip

\section{Related Work}

In the field of deep neural networks, the most commonly used optimizers are still hand-designed, such as momentum SGD, RMSProp, or Adam \cite{adam_opt}.
Tremendous attempts were made to improve their performance, or exceed them with different approaches.
Finding the best optimizer in general is not possible, see \cite{optimizers_no_free_lunch}. However by restricting the tasks, and hence, finding a better optimizer can be regarded as a learning process.
This learning can focus either within a single task, or across multiple tasks.

For the former, typical approaches include adaptive learning based methods \citep{HyperAdam, switch_from_Adam_to_SGD} and extensions of hand-crafted optimizers (e.g.: by replacing some components with RNN) \citep{learning_to_optimize_by_lr, learning_to_optimize_by_lr2, learning_to_optimize_primer_and_benchmark}. However these methods aim to speed up single task convergence, but adaptive learning may have some drawbacks, especially on  NLP tasks, and certain synthetic problems \cite{adaptive_methods_slow_convergence}.

The later one, namely learned optimization, trains a function (typically an auxiliary neural network) to optimize the model that solves the original tasks.
Some of them employ traditional gradient-descent based optimizers \citep{andrychowicz2016learning, metz2022velo, 155621, 10.5555/3305890.3306069}, other methods utilize evolutionary algorithms \citep{155621, metz2020tasks}, or reinforcement learning \citep{li2017learning, bello2017neural}.
These methods aim to discover an optimizer via some form of learning that enhance performance over a diverse set of tasks, surpassing the effectiveness of commonly used, hand-designed optimizers.

Our approach aligns more with the adaptive learning based methods. However, rather than focusing on performance enhancement for specific group of tasks, our objective is to establish a comprehensive mathematical framework,  which supports the combination of many existing optimizers, and enables the exploration of new optimization algorithms. 
While our goal is different, our work shows some similarities with learned-optimization as well.
Training an optimizer involves backpropagating through many unrolled inner optimization steps, a process that is not only computationally intensive but also susceptible to instability, as highlighted by \cite{metz2020tasks}.
Within this context, our RLLC method could be interpreted as a kind of differential-free optimization technique. In this paper, we focus solely looking one step ahead, but one could easily extend it beyond that.

\section{Retrospective Learning Law Correction}

\noindent{\bf Functional approach to optimizers:}~~The RLLC method is presented through an abstract mathematical framework for optimizers. This framework is somewhat specialized, yet it maintains sufficient generality to encompass a range of interesting optimizers. We think of optimizers as entities with an evolving internal state that updates at each step based on newly received gradients. Additionally, the optimizer calculates a parameter update vector relevant to the optimization process. A functional description of such an optimizer is given in the following definition.

\begin{definition}\label{funcopt} An {\it optimizer} for $n$ parameters is a pair of functions of the form $F:\mathcal{S}\times\mathbb{R}^n\to\mathcal{S}$ and $G:\mathcal{S}\times\mathbb{R}^n\to\mathbb{R}^n$ where $\mathcal{S}$ is the set of possible internal states, $F$ is the {\it state update function} and $G$ is the {\it parameter update function}. 
\end{definition}

To translate optimizers into an actual optimization process we choose an initial internal state $S_0\in\mathcal{S}$ and an initial parameter vector $\theta_0\in\mathbb{R}^n$. Then we iterate $$S_t:=F(S_{t-1},g_t),~~\theta_t:=\theta_{t-1}+G(S_t,g_t)$$ where $g_t$ is a gradient vector received by the optimizer in the $t$-th step. To illustrate this formalism, assume that the optimizer is given by $\mathcal{S}=\mathbb{R}^n$, $F(v,w)=\beta v+w, G(v,w)=-cv$. In this case, we obtain the momentum SGD with learning rate $c$ and decay parameter $\beta$. 

\noindent{\bf Optimizers with memory and RLLC:}~~We will think of memory $k$ optimizers in a way that the internal state space is of the form $\mathbb{R}^{n\times k}\times\mathcal{H}$ where the columns of matrices in $\mathbb{R}^{n\times k}$ represent $k$ vectors in the parameter space $\mathbb{R}^n$ and $\mathcal{H}$ will be called the space of hidden states. We typically assume that $n$ is a large number and that the hidden states are described by much fewer than $n$ parameters. A {\it memory update rule} for $k$ memory units is a function of the form $U:(\mathbb{R}^{n\times k}\times\mathcal{H})\times\mathbb{R}^n\to\mathbb{R}^{n\times k}\times\mathcal{H}$ where the external $\mathbb{R}^n$ component represents new arriving gradients. Such a function does not yet determine an optimizer. The RLLC method is designed to turn memory update rules into optimizers by extending their state space with a vector called learning law and introducing a natural parameter update function. We give two different descriptions of RLLC. The first one is a functional description which is more convenient for proofs. 

We will need the so-called Moore-Penrose inverse which is defined for an arbitrary matrix $A\in\mathbb{R}^{n\times k}$ and is denoted by $A^+$. Note that if $A$ has rank $k$ (which means that $A$ is non-degenerate if $n\geq k$) then $A^+=(A^TA)^{-1}A^T$.

\begin{definition}[RLLC functional form]\label{RLLCdef} Let $U$ be a memory update rule as above. Then the corresponding RLLC optimizer with learning rates $c_1,c_2$ is given as follows. The state space is $\mathcal{S}:=\mathbb{R}^{n\times k}\times\mathcal{H}\times\mathbb{R}^k$ where the extra component $\mathbb{R}^k$ is called the learning law. The functions $F,G$ are given in the following way. Assume that $M\in\mathbb{R}^{n\times k},H\in\mathcal{H},L\in\mathbb{R}^k,g\in\mathbb{R}^n$. Then
$$F(M,H,L,g):=(U_1(M,H),U_2(M,H),L+c_2M^+g)$$
$$G(M,H,L,g):=-c_1ML.$$ 
\end{definition}

\begin{remark} Vectors in $\mathbb{R}^m$ are considered to be column vectors. This means that they are treated as $m\times 1$ matrices in calculations.   
\end{remark}

The second, less abstract approach to RLLC describes the optimization process directly in a more conventional way.


\noindent{\bf Require:} 
\vspace{-0.2cm}
\begin{itemize}
\vspace{-0.2cm}\item $\theta_0$ initial parameter vector
\vspace{-0.2cm}\item $f(\theta)$: stochastic objective function with parameters $\theta\in\mathbb{R}^n$
\vspace{-0.2cm}\item Two learning rates $c_1,c_2>0$
\vspace{-0.2cm}\item Stability parameter $\epsilon$ for relaxed Penrose inverse
\vspace{-0.2cm}\item $M_0\in\mathbb{R}^{n\times k}$: initial memory units 
\vspace{-0.2cm}\item $L_0\in\mathbb{R}^k$: initial learning law
\vspace{-0.2cm}\item $H_0$: initial hidden state
\end{itemize}

$t:=0$: initialize time step

{\bf while:} $\theta_t$ not converged {\bf do}:

~~$t\longleftarrow t+1$

\medskip

\tikzstyle{mybox} = [draw=black, fill=blue!20, very thick,
    rectangle, rounded corners, inner sep=10pt, inner ysep=15pt]

\tikzstyle{mybox2} = [draw=black, fill=green!20, very thick,
    rectangle, rounded corners, inner sep=10pt, inner ysep=15pt]

\tikzstyle{mybox3} = [draw=black, fill=yellow!20, very thick,
    rectangle, rounded corners, inner sep=10pt, inner ysep=15pt]

\tikzstyle{mybox4} = [draw=black, fill=red!20, very thick,
    rectangle, rounded corners, inner sep=10pt, inner ysep=15pt]
    
\tikzstyle{fancytitle} =[fill=gray, text=white]

\begin{tikzpicture}
\node [mybox4] (box){%
    \begin{minipage}{0.40\textwidth}
~~$g_t:=\bigtriangledown_\theta f_t(\theta_{t-1})$ (Get gradients w.r.t. stochastic objective at timestep t)

    \end{minipage}
};
\node[fancytitle, right=10pt] at (box.north west) {Get New Gradient};
\end{tikzpicture}

\medskip

\vspace{-0.2cm}\begin{tikzpicture}
\node [mybox] (box){%
    \begin{minipage}{0.40\textwidth}

~~$L_t:=L_{t-1}+c_2M_{t-1}^+g_t$ 
    \end{minipage}
};
\node[fancytitle, right=10pt] at (box.north west) {RLLC Step};
\end{tikzpicture}

\medskip

\vspace{-0.2cm}\begin{tikzpicture}
\node [mybox2] (box){%
    \begin{minipage}{0.40\textwidth}
~~$(M_t,H_t):=U(M_{t-1},H_{t-1},g_t)$
    \end{minipage}
};
\node[fancytitle, right=10pt] at (box.north west) {Memory Update};
\end{tikzpicture}

\medskip

\vspace{-0.2cm}\begin{tikzpicture}
\node [mybox3] (box){%
    \begin{minipage}{0.40\textwidth}
~~$\theta_t:=\theta_{t-1}-c_1M_tL_t$ 
    \end{minipage}
};
\node[fancytitle, right=10pt] at (box.north west) {Learning Step};
\end{tikzpicture}

\bigskip

\noindent{\bf Explanation of RLLC and remarks:}

The idea behind the learning rule update is that with the arrival of the new gradient $g_t$ the optimizer gains new (retrospective) information on how it could have done better in the previous learning step. Notice that the vector $M^+g_t$ is the coefficient vector of the orthogonal projection of $g_t$ to the space spanned by the memory units when written as a linear combination of the memory units. This means that, if performed with the new law, the outcome of the parameter update in step $t-1$ would have been $\theta_t-c_1c_2p_t$ instead of $\theta_t$ where $p_t$ is the orthogonal projection of $g_t$ to the space spanned by the memory units in the $t-1$-th step. Notice that $(p_t,g_t)=(p_t,p_t)\geq 0$ and thus the change $-c_1c_2p_t$ points in a direction which improves the objective function. 

The above heuristics does not take it into account that the objective function $f_t$ is also changing. This fact indicates that our update rule is more justified if the second learning rate $c_2$ is small and thus random effects have time to average out leaving only useful directions in the update. Also notice that the algorithm does not "go back in time" to perform the improved learning step. Instead it applies the updated learning law with the updated memory units. This shows that the efficiency of the RLLC method depends on a type of consistency property. Roughly speaking it assumes that the notion of a "good learning law" does not change too much in time and so improvements of the past give improvements of the future. For this reason the choice of the memory update rule is a crucial issue which is one of the main topics of the second part of this paper. 

\begin{remark} The performance of the RLLC optimizer is dependent on the initialization of the learning law at the beginning. In practice it is not initialized to be $0$.
\end{remark}

\begin{remark} To avoid numerical instability, in practice we use a relaxed version of Penrose inverse which has a parameter $\epsilon$ set to a small number.    
\end{remark}

\noindent{\bf Linear invariance of RLLC:}~~We close this chapter with a useful linear invariance property of the RLLC method. We will need the next two definitions.

\begin{definition}\label{memconj} Let $U$ be a memory update rule as above and let $Q\in\mathbb{R}^{k\times k}$ be an invertible matrix. We define the new memory update rule $U^Q$ in the following way:
Let $M\in\mathbb{R}^{n\times k}$, $H\in\mathcal{H}$ and $U(MQ^{-1},H,g)=(M_2,H_2)$. Then 
$$U^Q(M,H,g):=(M_2Q,H_2).$$
\end{definition}

\begin{definition} Two optimizers given by $(F,G)$ and $(F',G')$ with state spaces $\mathcal{S}$ and $\mathcal{S}'$ are called equivalent if there is a bijection $\phi:\mathcal{S}\to\mathcal{S}'$ (called an {\it isomorphism}) such that $\phi(F(S,v))=F'(\phi(S),v)$ and $G(S,v)=G'(\phi(S),v)$. A {\it partial isomorphism} is a bijection  between a subset of $\mathcal{S}$ and a subset of $\mathcal{S}'$ having the same property. If there is such a function we say that the two optimizers are partially equivalent on these two subsets. In particular, if two optimizers with memory $k$ states are partially equivalent on states with rank $k$ memory matrices then we call them {\it essentially equivalent}.
\end{definition}

It is easy to see that if two optimizers are equivalent then they define the same optimization process if their initialization of internal states is isomorphic. If two optimizers are partially equivalent with partial isomorphism $\phi$ then the optimization processes are identical as long as they operate on states in the domain (and image) of $\phi$. 

\begin{lemma}[Linear invariance of RLLC]\label{RLLCinv} Let $U$ be a memory update rule as above and $Q\in\mathbb{R}^{k\times k}$ be an arbitrary matrix. Then the RLLC optimizer corresponding to $U$ is essentially equivalent to the RLLC optimizer corresponding to $U^Q$. 
\end{lemma}

\begin{proof} We claim that the function $\phi(M,H,L):=(MQ,H,Q^{-1}L)$ is a partial isomorphism on states with rank $k$ memory matrices. This follows trivially from formulas in definition \ref{RLLCdef} and the fact that $(MQ)^+=Q^{-1}M^+$ holds if ${\rm rank}(M)=k$.
\end{proof}

\section{Linear memory updates}\label{linmemupdt}

Throughout this chapter we investigate linear memory update rules with no hidden states. Such an update rule is given by 
\begin{equation}\label{linupdt}
U(M,g):=MB+ga^T
\end{equation}
where $M\in\mathbb{R}^{n\times k}$ is the memory matrix, $g\in\mathbb{R}^n$ is a new gradient and $a\in\mathbb{R}^k,B\in\mathbb{R}^{k\times k}$ are fixed parameters of the update rule. In an optimization process this means that the memory unit $m_i$ represented by the $i$-th column of $M$ is updated to
$a_ig+\sum_{j=1}^kB_{j,i}m_i$ when the new gradient $g$ is received.

\noindent{\bf Linear memory optimizers with fixed learning law:}~~Linearly updated memory units are interesting independently of the RLLC method. We can directly obtain powerful optimizers by using a fixed hand designed learning law $L\in\mathbb{R}^k$. This type of optimizer, denoted by $\mathcal{L}(B,a,L)$ works by the equations:
$$M_t=M_{t-1}B+g_ta^T~,~\theta_t=\theta_{t-1}-M_tL.$$ 
If $k=1$ then $B,a,L$ are single real numbers. The corresponding optimizer is a momentum SGD optimizer with decay parameter $B$ and learning rate $aL$. Another, important setting is described in the following lemma (for proof see \cref{SGD_Nesterov_proof}).

\begin{lemma} Let

$$B=
\begin{pmatrix}
\beta & 0\\
0 & 0 
\end{pmatrix},~ a=
\begin{pmatrix}
1 \\
1 
\end{pmatrix},~L=
\begin{pmatrix}
c\beta \\
c 
\end{pmatrix}
$$
Then the corresponding optimizer $\mathcal{L}(B,a,L)$ is Nestorv Accelerated Gradient with decay parameter $\beta$ and learning rate $c$.  
\end{lemma}

\noindent{\bf Abstract rule of a memory unit:}~~ There is a useful observation which sheds more light on what information linear memory stores if this memory update is iterated in an optimization process started with initial value $0^{n\times k}$ for $M$. By induction we have 
$$U(...U(U(0^{n\times k},g_1),g_2),...,g_t)=\sum_{i=1}^t g_ta^TB^{t-i}.$$
We obtain that at time $t$ (after $t$ iteration of the update rule) the value of the $i$-th memory unit is given by
\begin{equation}\label{abstrule}
m_j=\sum_{i=0}^{t-1} g_{t-i}(a^TB^i)_j
\end{equation}
where $(a^TB^i)_j$ denotes the $j$-th coordinate of the row vector $a^TB^i$. If we regard gradients with index $0$ or negative index as $0$ then the sum can be taken from $0$ to infinity. Informally speaking, this means that $m_i$ is a fixed (time independent) linear combination of previous gradients going backwards in time. 
This linear combination is represented by the infinite sequence $\{(a^TB^i)_j\}_{i=0}^\infty$ for the $j$-th memory unit. We say that this infinite sequence is the {\it abstract rule} of the memory unit. To guarantee that older gradients are taken with decaying weight in (\ref{abstrule}) we need to assume that the spectral norm of $B$ is smaller than $1$. 

\medskip

\noindent{\bf Real momentum:}~Memory updates in the case $k=1$ are determined by two numbers: $\beta=B_{1,1}$ and $\alpha=a_1$. The update rule of the single memory unit $m$ in the $t$-th step is $m\longleftarrow \alpha g_t+\beta m$. (This is essentially the update rule of a momentum vector.) It follow from our formula that the abstract rule in this case is given by the geometric sequence $\beta^i\alpha$. In particular, in the $t$-th step we have that $m=\sum_{i=0}^t \alpha\beta^i g_{t-i}$. 
\medskip

\noindent{\bf Complex momentum:}~The case 

\begin{equation}\label{comprep} B=\begin{pmatrix}
\alpha & -\beta\\
\beta & \alpha 
\end{pmatrix}
\end{equation}

has a distinguished role because such matrices 
represent complex numbers $\gamma=\alpha+\beta i$. In this special case the two memory units can be interpreted as the real and the complex parts of a single complex valued memory unit which describes a momentum vector with complex parameter $\gamma$. More precisely $m_1$ and $m_2$ are the real and the complex parts of a memory unit $m\in\mathbb{C}^n$ which is updated according to $m\longleftarrow g_t+\gamma m$.
 
\medskip

\noindent{\bf Jordan block of size $2$:}\label{jordan_prop} ~Another interesting example for $k=2$ is given by 
$$B=\begin{pmatrix}
\alpha & 1\\
0 & \alpha 
\end{pmatrix}$$
which is the so called Jordan block of size $2$ with eigenvalue $\alpha$. The first one of the two memory units is the momentum vector of the gradients with parameter $\alpha$. However the second memory unit stores something new. It is the momentum vector of the first memory unit with parameter $a$. One can show that the abstract rule corresponding to this memory unit is given by the infinite sequence $0,1,2\alpha,3\alpha^2,4\alpha^3,\dots$. 

\medskip

\noindent{\bf Propagators and their unions:}~In general, we call $k$ memory units $m_1,m_2,\dots,m_k$ connected by a joint linear update rule a {\it propagator} of dimension $k$. Recall that such an object is described by a matrix $B\in\mathbb{R}^{k\times k}$ 
 and a vector $a\in\mathbb{R}^k$. In the previous two examples each real number $\beta$ is associated with a propagator denoted by $M(\beta)$ of dimension $1$ while complex numbers $\gamma$ are associated with a propagator denoted by $CM(\gamma)$ of dimension $2$. We call such propagators {\it momentum propagators}. It will be important for us that there is a simple operation on propagators that we call {\it union} and denote by $\oplus$. This is simply just taking the union of the corresponding memory units together and updating them independently.  From a linear algebraic point of view, the matrix $B$ corresponding to the union of propagators is a block diagonal matrix whose blocks contain the matrices of the individual propagators. The vector $a$ corresponding to the union is the concatenation of the vectors of the propagators. Unions of momentum propagators will be called {\it multi momentum propagators}.

\medskip

\section{Optimizers with linear memory and RLLC}

In this chapter we discuss the basic properties of optimizers which combine linear memory and the RLLC method. We use the term LM-RLLC optimizers for them. Based on definition \ref{RLLCdef} and formula (\ref{linupdt}) one can produce the LM-RLLC optimizer $\mathcal{F}(B,a,c_1,c_2)$ with hyperparameters $B,a,c_1,c_2$. The corresponding update functions are given by $$F(M,L,g)=(MB+ga^T,L+c_2M^+g)$$ $$G(M,L,g)=-c_1ML$$ For the sake of completeness we describe the recursive optimization process.

\begin{definition}(LM-RLLC optimization process) Let us fix the hyperparameters $B\in\mathbb{R}^{k\times k},a\in\mathbb{R}^k,(c_1,c_2)\in\mathbb{R}^2.$ Then the LM-RLLC optimizer with these hyperparameters is given by the equations
$$L_t=L_{t-1}+c_2M_{t-1}^+g_t$$
$$M_t=M_{t-1}B+g_ta^T$$ 
$$\theta_t=\theta_{t-1}-c_1M_tL$$
where $M_0$ is the $0$ matrix in $\mathbb{R}^{n\times k}$ and $L_0\in\mathbb{R}^k$ is a suitable (typically non $0$) vector.
\end{definition}

Deeper mathematical analysis reveals that LM-RLLC optimizers can be transformed into a simpler, canonical form if we look at them up to equivalence. The key observation is a "basis independence" property of LM-RLLC optimizer functions.

\begin{theorem}[Basis independence of LM-RLLC optimizers]\label{LM-RLLCinv} Let $k\in\mathbb{N},a\in\mathbb{R}^k,B\in\mathbb{R}^{k\times k},(c_1,c_2)\in\mathbb{R}^2$ and let $Q\in\mathbb{R}^{k\times k}$ be an invertible matrix. Then $\mathcal{F}(B,a,c_1,c_2)$ is essentially equivalent to $$\mathcal{F}(Q^{-1}BQ,Qa,c_1,c_2).$$  
\end{theorem}

\begin{proof} The optimizer $\mathcal{F}(B,a,c_1,c_2)$ is obtained from the linear memory update rule $U(M,g)=MB+ga^T$ with RLLC. Notice that $U^Q$ (in the sense of definition \ref{memconj}) is given by $U^G(M,g)=M(Q^{-1}BQ)+g(Qa)^T$. Then lemma \ref{RLLCinv} finishes the proof.   
\end{proof}

\noindent{\bf Real Jordan normal form:}~~Theorem \ref{LM-RLLCinv} together with a variant of the Jordan decomposition theorem implies that we can transform LM-RLLC optimizers into a very special form without changing the optimization process. The original form of the Jordan decomposition theorem says that if $B\in\mathbb{C}^{k\times k}$ is an arbitrary complex matrix then there is an invertible matrix $Q\in\mathbb{C}^{k\times k}$ such that $Q^{-1}BQ$ has a block diagonal form with each block being a so-called Jordan block. A Jordan block $J_m(\lambda)$ is a matrix of size $m\times m$ with $\lambda\in\mathbb{C}$ in the diagonal, $1$ above the diagonal and $0$ everywhere else. For example

$$ J_3(\lambda)=
\begin{pmatrix}
\lambda & 1         & 0   \\
0         & \lambda & 1   \\
0         & 0         & \lambda \\
\end{pmatrix}
$$
There is a similar, although somewhat more complicated statement (called real Jordan normal form) if $B$ and $Q$ are required to be real matrices. In this case there are two types of blocks $J_m(\lambda)$ with $\lambda\in\mathbb{R}$ and $CJ_m(\alpha+\beta i)$ with $\alpha,\beta\in\mathbb{R}$. The second type of block has size $2m\times 2m$ and it "imitates" complex Jordan blocks with real matrices. This matrix is very similar to $J_m(\lambda)$ with the main difference being that each entry is replaced by a $2\times 2$ matrix. The $0$'s and $1's$ are replaced by $0$ matrices and identity matrices. The $\lambda$ entries are replaced by the matrix in equation $(\ref{comprep})$ which represents $\alpha+\beta i$ by a real matrix. 
For example 
$$ CJ_2(\alpha+\beta i)=
\begin{pmatrix}
\alpha & -\beta         & 1 & 0  \\
\beta  & \alpha & 0 & 1  \\
0      & 0      & \alpha & -\beta \\
0      & 0      & \beta  & \alpha
\end{pmatrix}
$$

\noindent{\bf Propagators of Jordan type:}~~Racall that an LM-RLLC optimizer is given by $B\in\mathbb{R}^{k\times k},a\in\mathbb{R}^k$ and two learning rates. By transforming the matrix $B$ to its real Jordan normal form with a basis transformation given by $Q\in\mathbb{R}^{k\times k}$, we can divide the memory units into groups belonging to single blocks of type $J_m(\lambda)$ or $CJ_m(\alpha+\beta i)$. The block diagonal form of $Q^{-1}BQ$ guarantees that these groups do not interact with each other in memory updates and thus we can treat them as separate propagators. Recall that in this basis transformation considered in theorem \ref{LM-RLLCinv} the vector $a$ transforms into $Qa$. By applying a statement which is slightly stronger then the Jordan decomposition theorem we can also guarantee that the part of $a$ in each block contains at most one coordinate with $1$ and the rest is $0$. We can also assume that this coordinate is the first one otherwise there are trivial memory units which store $0$ in each step. By summarizing all of this we obtain propagators of very special type. Let $e_m\in\mathbb{R}^m$ denote the vector with $1$ in the first coordinate and $0$ in the rest. We denote the propagator corresponding to the pair $(J_m(\lambda),e_m)$ by $M_m(\lambda)$ and the propagator corresponding to $(CJ_m(\alpha+\beta i),e_{2m})$ by $CM_m(\alpha+\beta i)$. We call such propagators {\it Jordan block propagators}. If $m=1$ then we omit the index and simply write $M(\lambda)$ and $CM(\alpha+\beta i)$. We obtain the next theorem.

\begin{theorem}[Normal forms of LM-RLLC optimizers] Every LM-RLLC optimizer is essentially equivalent with another LM-RLLC optimizer where the memory update is of the form $P_1\oplus P_2\oplus\dots\oplus P_r$ where each $P_i$ is a Jordan block propagator.
\end{theorem}

By slightly abusing the notation we will also use the formula $P_1\oplus P_2\oplus\dots\oplus P_r$ for the optimizer itself. For example $M(0.9)\oplus M_2(0.6)\oplus CM_2(0.3+0.2i)$ stands for a memory $7$ optimizer where the memory units are grouped and updated according to the propagators $M(0.9)$,$M_2(0.6)$ and $CM_2(0.3+0.2i)$.

\section{Experiments}

For our experiments, we used the Learned Optimization framework \cite{metz2022practical} as a starting point. The framework offers pre-trained and hyper parameter optimized optimizers. 
We compare our results with the most widely used optimizers as baseline: {\it Adam}, {\it SGD}, and {\it SGD with momentum}. We compare test loss and classification accuracy on MNIST\cite{deng2012mnist}, Fashin-MNIST\cite{xiao2017fashionmnist}, and CIFAR-10\cite{Krizhevsky2009LearningML} datasets. We experimented with dense, convolutional and residual neural networks.
The source code of our work is available publicly\footnote{\url{https://anonymous.4open.science/r/easytrace-843E/README.md}}. See implementation details in \cref{appendix:impleementation_details}.

\begin{table*}[!htb]
\fontsize{9pt}{9pt}\selectfont
\centering
\begin{tabular}{rcccccccccccccccl}\toprule
& \multicolumn{2}{c}{\textbf{MNIST}} && \multicolumn{2}{c}{\textbf{Fashion-MNIST}} && \multicolumn{3}{c}{\textbf{CIFAR-10}} & \\
\cmidrule(lr){2-3}
\cmidrule(lr){5-6}
\cmidrule(lr){8-10}
& \textbf{MLP} & \textbf{Conv} && \textbf{MLP} & \textbf{Conv} && \textbf{MLP} & \textbf{Conv} & \textbf{ResNet-20} \\
\midrule

\multirow{2}{*}{$SGD$} & 0.0882 &   0.0331 & & 0.3426 & 0.3673 & & 1.4084 & 0.8359 & 0.6010 & Loss\\
 & 98.16 & 98.56 & & 88.65 & 86.97 & & 52.18 & 71.37 & 80.93 & Acc\\
\addlinespace
\multirow{2}{*}{$Momentum SGD$} & 0.0856 & 0.0324 & & 0.3476 & 0.2732 & & 1.4108 & \textcolor{BlueViolet}{0.7850} & \textcolor{BlueViolet}{0.5757} & Loss\\
 & \textcolor{BlueViolet}{98.22} & 98.97 & & 88.67 & \textcolor{BlueViolet}{90.78} & & 51.99 & 73.17 & 81.36 & Acc\\
\addlinespace
\multirow{2}{*}{$Adam$} & \textcolor{BlueViolet}{\textbf{0.0758}} & 0.0304 & & \textcolor{BlueViolet}{0.3407} & \textcolor{BlueViolet}{0.2704} & & \textcolor{BlueViolet}{1.3858} & 0.7920 & 0.5857 & Loss\\
 & 97.83 & \textcolor{BlueViolet}{98.99} & & \textcolor{BlueViolet}{88.78} & 90.76 & & \textcolor{BlueViolet}{52.43} & \textcolor{BlueViolet}{73.70} & \textcolor{BlueViolet}{81.48} & Acc\\
\addlinespace
\arrayrulecolor{gray!50}
\hline
\arrayrulecolor{black}
\addlinespace
\multirow{2}{*}{$M(0.9)$} & 0.0844 & 0.0310 & & 0.3408 & \textcolor{OliveGreen}{0.2661} & & 1.4021 & 0.8030 & \textcolor{OliveGreen}{0.5301} & Loss\\
 & 98.21 & \textcolor{OliveGreen}{99.03} & & 88.64 & \textcolor{OliveGreen}{90.96} & & 52.13 & \textcolor{OliveGreen}{73.95} & \textcolor{OliveGreen}{83.08} & Acc\\
\addlinespace
\multirow{2}{*}{$M(0.9){\oplus}M(0.0)$} & 0.0888 & 0.0323 & & 0.3475 & \textcolor{OliveGreen}{0.2678} & & 1.3973 & 0.7977 & \textcolor{OliveGreen}{0.5353} & Loss\\
 & \textcolor{OliveGreen}{\textbf{98.26}} & 98.95 & & \textcolor{OliveGreen}{88.82} & \textcolor{OliveGreen}{90.98} & & 51.71 & \textcolor{OliveGreen}{74.11} & \textcolor{OliveGreen}{\textbf{83.38}} & Acc\\
\addlinespace
\multirow{2}{*}{$M(0.9){\oplus}M(0.8){\oplus}M(0.7)$} & 0.0829 & 0.0343 & & \textcolor{OliveGreen}{0.3359} & \textcolor{OliveGreen}{\textbf{0.2563}} & & 1.4142 & \textcolor{OliveGreen}{0.7734} & \textcolor{OliveGreen}{\textbf{0.5268}} & Loss\\
 & \textcolor{OliveGreen}{98.23} & 98.95 & & 88.67 & \textcolor{OliveGreen}{91.11} & & 51.55 & \textcolor{OliveGreen}{75.42} & \textcolor{OliveGreen}{83.19} & Acc\\
\addlinespace
\multirow{2}{*}{$M_2(0.6)$} & 0.0801 & 0.0800 & & \textcolor{OliveGreen}{\textbf{0.3220}} & 0.4488 & & \textcolor{OliveGreen}{\textbf{1.3444}} & 1.0404 & 0.5811 & Loss\\
 & 98.17 & 97.60 & & 88.75 & 84.31 & & \textcolor{OliveGreen}{\textbf{53.42}} & 63.59 & 81.35 & Acc\\
\addlinespace
\multirow{2}{*}{$M(0.9){\oplus}M_2(0.6)$} & 0.0861 & 0.0319 & & 0.3536 & \textcolor{OliveGreen}{0.2636} & & 1.4155 & \textcolor{OliveGreen}{0.7602} & \textcolor{OliveGreen}{0.5354} & Loss\\
 & 98.21 & \textcolor{OliveGreen}{98.99} & & \textcolor{OliveGreen}{89.03} & \textcolor{OliveGreen}{90.95} & & 52.08 & \textcolor{OliveGreen}{\textbf{75.87}} & \textcolor{OliveGreen}{82.84} & Acc\\
\addlinespace
\multirow{2}{*}{$M(0.9){\oplus}M(0.0){\oplus}M_2(0.6)$} & 0.0877 & 0.0287 & & 0.3498 & \textcolor{OliveGreen}{0.2596} & & 1.4028 & \textcolor{OliveGreen}{\textbf{0.7216}} & \textcolor{OliveGreen}{0.5393} & Loss\\
 & \textcolor{OliveGreen}{98.23} & \textcolor{OliveGreen}{\textbf{99.03}} & & \textcolor{OliveGreen}{89.14} & \textcolor{OliveGreen}{\textbf{91.29}} & & 51.83 & \textcolor{OliveGreen}{75.76} & \textcolor{OliveGreen}{82.73} & Acc\\
\addlinespace
\multirow{2}{*}{$M_3(0.6)$} & 0.0797 & 0.0539 & & \textcolor{OliveGreen}{0.3282} & 0.3735 & & \textcolor{OliveGreen}{1.3798} & 0.9433 & \textcolor{OliveGreen}{0.5445} & Loss\\
 & 98.22 & 98.31 & & \textcolor{OliveGreen}{\textbf{89.25}} & 87.08 & & \textcolor{OliveGreen}{53.37} & 66.68 & \textcolor{OliveGreen}{82.43} & Acc\\
\addlinespace
\multirow{2}{*}{$M(0.9){\oplus}M(-0.9){\oplus}CM(0.9i)$} & 0.0873 & 0.0334 & & 0.3671 & \textcolor{OliveGreen}{0.2624} & & 1.4029 & \textcolor{OliveGreen}{0.7647} & \textcolor{OliveGreen}{0.5337} & Loss\\
 & 98.09 & 98.97 & & 88.57 & \textcolor{OliveGreen}{91.06} & & 52.10 & \textcolor{OliveGreen}{75.44} & \textcolor{OliveGreen}{83.07} & Acc\\

\bottomrule
\end{tabular}
\caption{Loss and accuracy are reported across three different datasets, using three distinct network architectures. The first three rows are dedicated to benchmark optimizers, whereas the subsequent rows showcase our results. The best benchmark result for each task (dataset and architecture pair) are highlighted in blue. Instances where our optimizer exceeds the best baseline result are marked in green. Additionally, the absolute best value for each task is emphasized in bold font. The results are the average of 3 runs with different random seeds.}
\label{table:main_results}
\end{table*}


\begin{figure}[H]
\includegraphics[width=\textwidth/2]{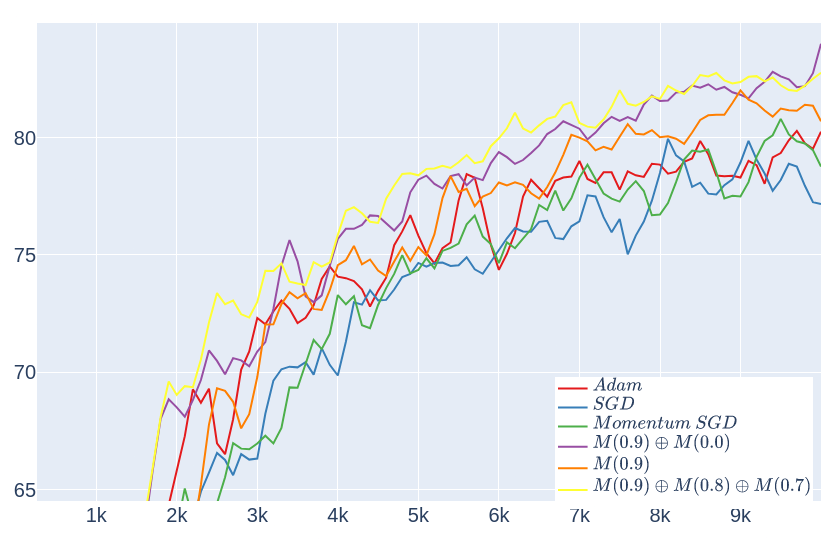}
\caption{
Test accuracy graphs of RLLC and benchmark optimizers, measured on the CIFAR-10 dataset, with the ResNet-20 network. RLLC optimizers show faster convergence and better generalization. See related plots and error bars in \cref{appendix:Supplementary_plots}.}
\label{fig:resnet_acc}
\end{figure}


\textbf{RLLC based adaptive learning rate:}~~ One of the simplest case of the RLLC method is already interesting. If there is a single memory unit containing the momentum of previous gradients then RLLC yields an adaptive version of the momentum SGD optimizer. In this case the learning law contains a single coefficient, that defines an adaptively changing learning rate for the momentum SGD. Our experiments show that this upgrade outperforms the plain momentum SGD method, showcasing the power of RLLC. See $M(0.9)$ results in \cref{table:main_results} and on \cref{fig:resnet_acc}. Note that RLLC can be applied to an arbitrary optimizer by introducing a single memory unit storing the last learning step. In a similar way we obtain a version of the optimizer with an adaptive learning rate. However it may depend on the optimizer whether it leads to a performance boost or not.

\textbf{Mixing SGD and momentum SGD:}~~ 
We observe an intriguing phenomenon when we enhance the memory unit of the previous method with the current gradient and monitor the learning law throughout the training process.
As shown in \cref{fig:SGD_MOM_weights}, during the initial phase of training, the coefficient of the $M(0.9)$ memory unit is predominant. However, as training advances, the coefficient of the $M(0)$ unit increases, leading to a reversal in the significance of the two memory units.
Our experiment supports \cite{SGD_MOM_paper} findings. \cref{table:main_results} and \cref{fig:resnet_acc} show results for $M(0.9) \oplus M(0)$.
An interesting additional detail is that in between the two extremal phases there exists a phase which emulates the Nesterov Accelerated Gradient (NAG) method. 
This occurs when the coefficient of the Momentum SGD memory unit, divided by the coefficient of the SGD unit, equals the decay parameter of the momentum SGD. (For more details, see \cref{SGD_Nesterov_proof}).

\begin{figure}[ht]
\includegraphics[width=8cm]{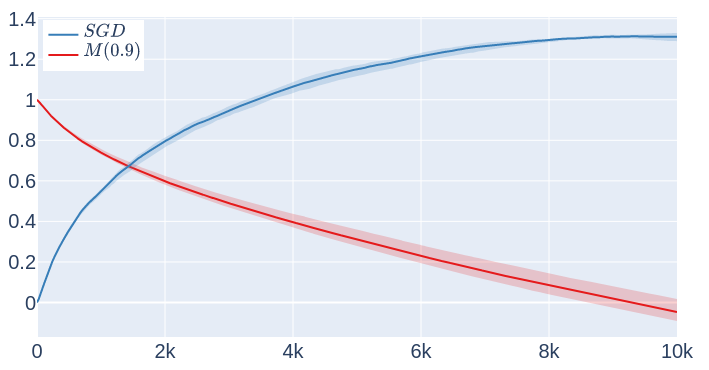}
\caption{
Analysis of $M(0.9) \oplus M(0.0)$ optimizer's memory unit's coefficients over time. 
The figure illustrates the optimizer's transition between momentum SGD and SGD, briefly aligning with the NAG optimizer around the 2k step.}
\label{fig:SGD_MOM_weights}
\end{figure}

\textbf{Multi-momentum propagators:}~~ In our experiment we investigated optimizers of the form $M(\beta_1)\oplus M(\beta_2)\oplus\dots\oplus M(\beta_k)$. We have not optimized the hyperparameters $\beta_i$ but we found very promising settings with a few trials. Our results are therefore illustrative and the fine tuning (depending on the type of network) is subject to 
further research. In \cref{table:main_results} we describe our experiments with $M(0.9)\oplus M(0),M(0.9)\oplus M(0.6),M(0.9) \oplus M(0.8) \oplus M(0.7)$ and $M(0.9)\oplus M(0.6)\oplus M(0)$ optimizers. Quite surprisingly the simplest one $M(0.9)\oplus M(0)$ (which is mentioned earlier) is the most reliable. However on certain tasks it is outperformed by the memory $3$ settings. \cref{fig:resnet_metaW} illustrates an interesting coupling between the coefficient of $M(0.8)$ and $M(0.7)$ memory units. See further details in \cref{appendix:additional_math_obs}. 

\begin{figure}[ht]
\includegraphics[width=8cm]{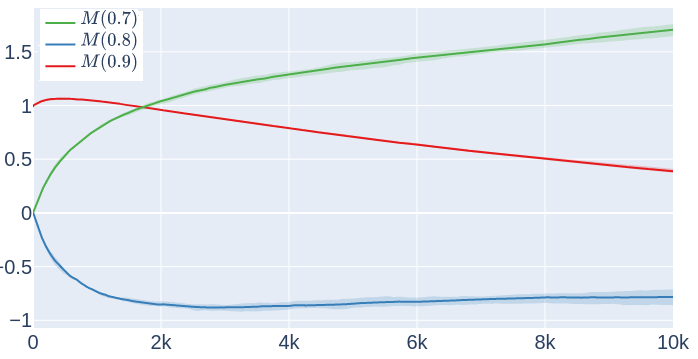}
\caption{
Analysis of the coefficients of $M(0.9) \oplus M(0.8) \oplus M(0.7)$ optimizer over time. The figure shows an interesting negative coupling between $M(0.8)$ and $M(0.7)$. See further details in \cref{appendix:additional_math_obs}.}
\label{fig:resnet_metaW}
\end{figure}

\textbf{$M_m(\lambda)$ propagator for $m\geq 2$:}~~
Jordan block propagators of the form $M_m(\beta)$ and $CM_m(\beta)$ with $m\geq 2$ are easy to implement in our code. In our experiments we focused on type $M_m(\beta)$ propagators with $m=2,3$. It is also interesting to combine them with other propagators. \cref{table:main_results} shows results for $M_2(0.6)$, $M_3(0.6)$ and $M(0.9) \oplus M_2(0.6)$. These configurations surpass the baseline optimizers in many tasks and also surpass pure multi-momentum propagators in some specific tasks.

\textbf{Complex-moment propagators:}~~
Another interesting possibility in our framework is the usage of complex-momentum propagators. One particular example that we experimented with is the case of $M(\beta){\oplus}M(-\beta){\oplus}CM(\beta i)$. This choice in not random. It comes from the Jordan normal form of the permutation matrix corresponding to the cyclic shift on $4$ elements multiplied with $\beta$. This particular propagator is closely related to Fourier analysis. 

\section{Limitations}
Using memory units comes at a cost. Each memory unit is a vector in the parameter space $\mathbb{R}^n$. In our experiments, we opted for relatively small or medium-sized architectures. However, for architectures with a vast parameter space, our approach with many memory units could prove to be too memory-intensive. 

It's also worth noting that our experiments were conducted on relatively small datasets, and future work should explore experiments on larger datasets.

\section{Conclusion} Our experiments demonstrate that the RLLC method is capable of boosting the performance of classical optimizers (such as SGD and momentum SGD) by combining them and making them more adaptive. Furthermore the case of linearly updated memory units provides a mathematically elegant framework with many new types of promising optimizers such as the ones corresponding to larger Jordan blocks, complex numbers and their combinations. We regard this paper as a starting point for future research in the frame of which the full potential of our approach is explored. One possible research direction is to introduce adaptively changing memory update rules. In particular, in the linear setting the pair $B\in\mathbb{R}^{k\times k},a\in\mathbb{R}^k$ (see \cref{linmemupdt}) is fixed for the whole optimization process in the current version. It would be interesting to study a version were $B$ and $a$ are also adaptively changing throughout learning.

\section{Acknowledgements} The first-named author received funding from the project KPP 133921 of the Hungarian Government. The research was also supported
by the Hungarian Ministry of Innovation and
Technology NRDI Office in the framework of the Artificial Intelligence National Laboratory
Program.

%% file: appendix.tex
\section{Interpolations between SGD, momentum SGD and Nesterov Accelerated Gradient (NAG)}
\label{SGD_Nesterov_proof}

As we have already explained, momentum SGD method in our interpretation is represented as the propagator $M(\beta)$ where $0\leq \beta<1$ is the decay parameter. In particular $M(0)$ corresponds to a memory unit which stores the last gradient seen by the optimizer. In this sense, a memory $1$ optimizer with memory unit $M(0)$ is basically an SGD optimizer. The learning law in this case is a single real number which manifests as a learning rate.  

If an optimizer has two memory units $M(\beta)$ and $M(0)$ then by changing the learning law (described by a pair of real numbers) we can continuously interpolate between momentum SGD and pure SGD. In this section we prove that there is a third well known optimizer which can be represented in this setting: the so-called Nesterov accelerated gradient (NAG). In the next list we summarize the meanings of special learning laws for $M(\beta),M(0)$.
\begin{enumerate}
\item $(0,c)$ : SGD 
\item $(c,0)$ : momentum SGD with decay parameter $\beta$ 
\item $(c\beta,c)$ : NAG with decay parameter $\beta$
\end{enumerate}
where the learning rate is $c*{\rm lr}$. Notice that since $\beta$ is a prescribed fix number, the above three cases don't cover the possible learning laws for the pair $M(\beta),M(0)$. 

The fact that the last learning law (if kept fixed!) is equivalent with NAG is an easy calculation from the equations describing the NAG optimizer: 

$$\phi_{t+1}=\theta_t-r\nabla f(\theta_t)$$
$$\theta_{t+1}=\phi_{t+1}+\beta (\phi_{t+1}-\phi_t).$$

Let 
$$v_{t+1}:=\phi_t-\phi_{t+1}.$$
With this notation we have the update rules:
$$v_{t+1}=\beta v_t+r\nabla f(\theta_t),$$
$$\theta_{t+1}=\theta_t-(r\nabla f(\theta_t)+\beta v_{t+1}).$$
Observe that if we introduce the update rule
$$m_{t+1}:=\beta m_t+\nabla f(\theta_t)$$ then $v_t=rm_t$ holds at any given time $t$. Furthermore $m_t$ corresponds to the propagator $M(\beta)$.
With this notation we have that
$$\theta_{t+1}=\theta_t-r(\nabla f(\theta_t)+\beta m_{t+1}).$$
Assume now that $r=c*{\rm lr}$ where ${\rm lr}$ is a fixed learning rate. Then we have that
$$\theta_{t+1}=\theta_t-{\rm lr}(c\nabla f(\theta_t)+c\beta m_{t+1})$$. 
This verifies our claim since $\nabla f(\theta_t)$ is the propagator $M(0)$.

\section{Additional mathematical observations}
\label{appendix:additional_math_obs}

\textbf{On the memory units of $M_k(\lambda)$ propagators:}~~ The $M_k(\lambda)$ propagator can naturally be interpreted as an iterated momentum propagator. Let $m_1,m_2,\dots,m_k$ denote the the memory units. The update rule of $M_k(\lambda)$ is given by
$$m_1\leftarrow m_1\beta+g$$
$$m_i\leftarrow m_i\beta+m_{i-1}~{\rm for}~i\geq 2.$$ Thus $m_1$ is the momentum vector of the gradient and $m_i$ (for $i\geq 2$) is the momentum vector of $m_{i-1}$.
One can compute that the abstract rule of $m_i$ is given by the sequence 
$\{\beta^{j-i+1}{{j}\choose{i-1}}\}_{j=0}^\infty.$ It follows that the subspace generated by the abstract rules of the memory units is the space of all sequences of the form $\{p(j)\beta^j\}_{j=0}^\infty$ where $p$ is a polynomial of degree at most $k-1$. This means that we can associate such a polynomial with each learning law. The RLLC method for $M_k(\lambda)$ basically adaptively navigates in this polynomial space.

\medskip

\textbf{Relation between Multi-momentum and $M_k(\lambda)$ propagators:}~~
The progression of the learning law of $M(0.9) \oplus M(0.8) \oplus M(0.7)$ presents an interesting phenomenon. As \cref{fig:resnet_metaW} shows, the coefficients of $M(0.8)$ and $M(0.7)$ memory units are noticeably coupled with opposite sign. One might assume that the algorithm is just trying to cancel their effects, but the performance improvement compared to $M(0.9)$ suggests that something more interesting is happening here. A deeper explanation relates this optimizer to another one of the form $M(0.9)\oplus M_2(0.75)$. More precisely if we consider $\mathcal{O}(\epsilon)=M(0.9)\oplus M(0.75+\epsilon)\oplus M(0.75-\epsilon)$ we find that as $\epsilon$ goes to $0$ the subspace spanned by the abstract rules of the memory units converges to the subspace spanned by the memory units of $M(0.9)\oplus M_2(0.75)$. In theory this convergence means that the optimizers themselves converge. Note that in practice we can not model $M(0.9)\oplus M_2(0.75)$ by $\mathcal{O}(\epsilon)$ because numerical instability arises if $\epsilon$ is very small.


\section{Implementation details}
\label{appendix:impleementation_details}

\subsection{Network architectures and training details}

\textbf{Dense network}
Our dense network comprises three hidden layers, each with a width of 128 and followed by a ReLU activation function. We did not include any normalization layers.

\textbf{Convolutional network}
Our convolutional network features a depth of three, with channel widths of 32, 64, and 64, each followed by a ReLU activation function. We did not incorporate any normalization layers. Following the convolutional layers, we apply max pooling and then a final dense layer.

\textbf{ResNet-20}
Our ResNet-20 variant adheres to established conventions for CIFAR-10, employing a three-level architecture with three residual blocks at each level. Each residual block is composed of the following sequence of layers: Convolution-Batch Normalization-ReLU-Convolution-Batch Normalization. A ReLU operation is applied after the addition operation in each residual block. The convolution kernels are 3x3 in size.

\subsection{Training details, hyperparameters}

In all reported experiments, we employed a batch size of 128 and trained the models for 10,000 iterations. We did not use a learning rate scheduler, to avoid any potential variance in its effect across different optimizers. We run every experiment with 3 different seed, and reported the average of the results.

During our hyperparameter optimization process, we tested the following potential values:
\begin{itemize}
  \item \textbf{Benchmark optimizers}
    \subitem learning rate: 1e-7, 3e-7, 1e-6, 3e-6, 1e-5, 3e-5, 1e-4, 3e-4, 1e-3, 3e-3, 1e-2, 3e-2, 1e-1, 3e-1, 1
  \item \textbf{Our optimizers}
    \subitem learning rate: 0.001, 0.003, 0.01, 0.03, 0.1, 0.3
    \subitem learning law - learning rate: 0.003, 0.01, 0.03
\end{itemize}

\subsection{Trainind datasets}

\textbf{CIFAR-10} The CIFAR-10 dataset \cite{cifar10_dataset} consists of $60000$ $32x32$ colour images in $10$ classes, with $6000$ images per class.
There are $50000$ training images, and $10000$ test images.
We used the canonical train–validation-test split, with $45000$ train, $5000$ validation, and $10000$ test images.
As a preprocessing, we normalized the images with the means $(0.4914, 0.4822, 0.4465)$ and standard deviations $(0.2023, 0.1994, 0.2010)$ for the three RGB channels, respectively.
On the \textit{ResNet} task we used random resized crop (with zoom scale $0.8$-$1.2$), horizontal flip, and random rotation.

\textbf{Fashion-MNIST} The Fashion-MNIST dataset \cite{fashion_mnist_dataset} consists of $70000$ $28x28$ monochrom images in $10$ classes, with $7000$ images per class. There are $60000$ training images, and $10000$ test images.
We used the canonical train–validation-test split, with $54000$ train, $6000$ validation, and $10000$ test images.
As a preprocessing, we normalized the images with the mean $0.3$ and stadard deviation $0.3$.

\textbf{MNIST} The MNIST dataset \cite{lecun2010mnist} consists of $60000$ $28x28$ monocrom images in $10$ classes, with $7000$ images per class.
There are $60000$ training images, and $10000$ test images.
We used the canonical train–validation-test split, with $54000$ train, $6000$ validation, and $10000$ test images.
As a preprocessing, we normalized the images with the mean $0.1307$ and standard deviation $0.3081$.

\section{Computational resources}
For our experiments we used a server with 8 A10040GB GPUs. We reported outcomes from a total of 16 optimizers, each optimized for hyperparameters across seven distinct tasks  (architecture - dataset pair). 
One round of hyperparameter optimization with three different random seeds took approximately 4 hours on our server, and we were able to run 16 paralelly.

Therefore, all of our results can be replicated in about 28 hours using the same setup, or in 224 hours on a single A100 40GB GPU.

\section{Supplementary Plots}
\label{appendix:Supplementary_plots}
\cref{fig:Supplementary_plots} shows additional accuracy plots for MLP and Convolutional tasks.
\cref{fig:error_bar} demonstrates, that the test accuracy is consistent on different random seeds for the demonstrated experiments.

\begin{figure}
\includegraphics[width=\textwidth]{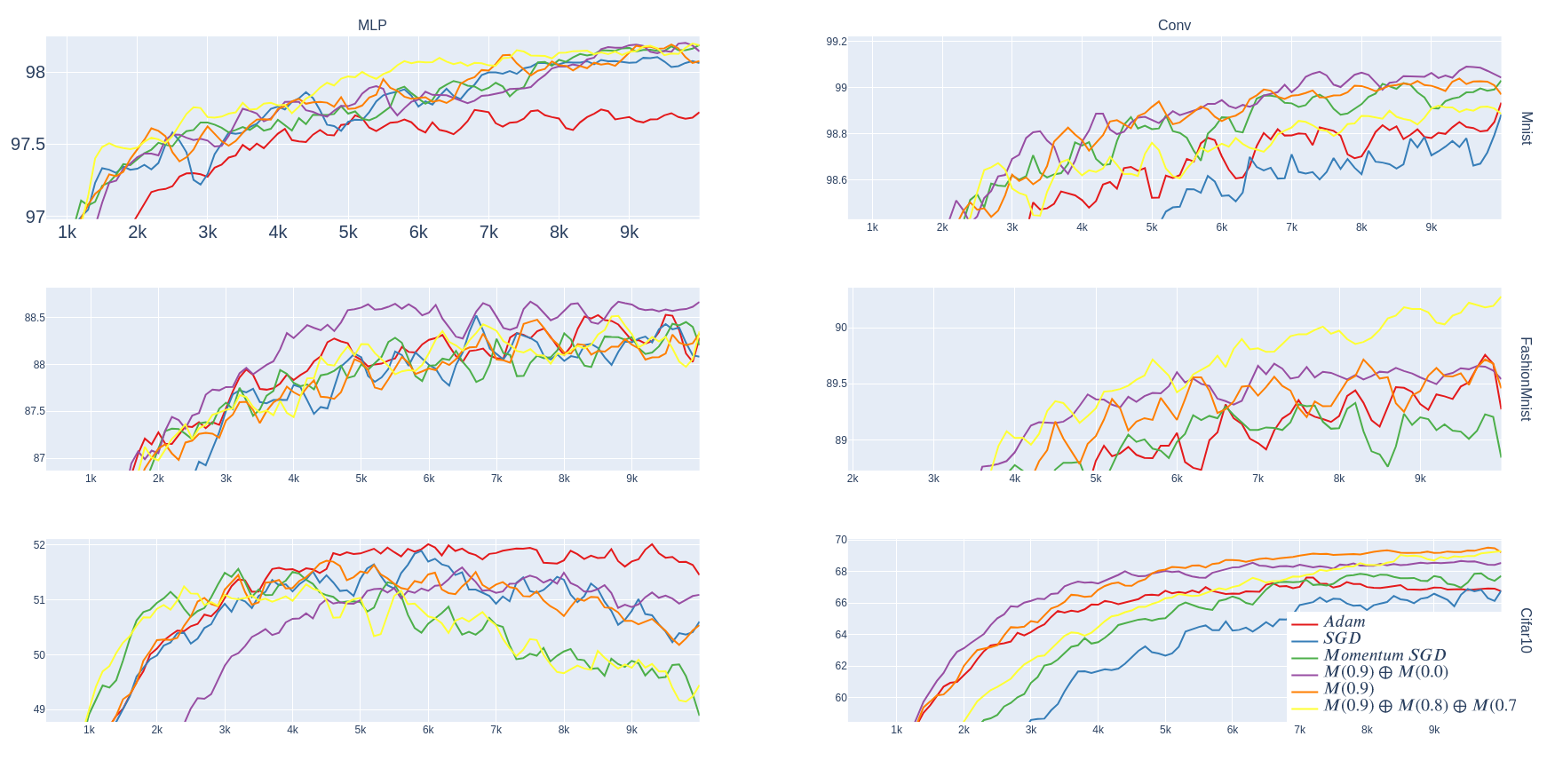}
\caption{Test accuracy graph of some RLLC optimizer, comparing with benchmark optimizers. On most of the tasks RLLC optimizers perform better, than the benchmark optimizers.}
\label{fig:Supplementary_plots}
\end{figure}

\begin{figure}
\includegraphics[width=\textwidth]{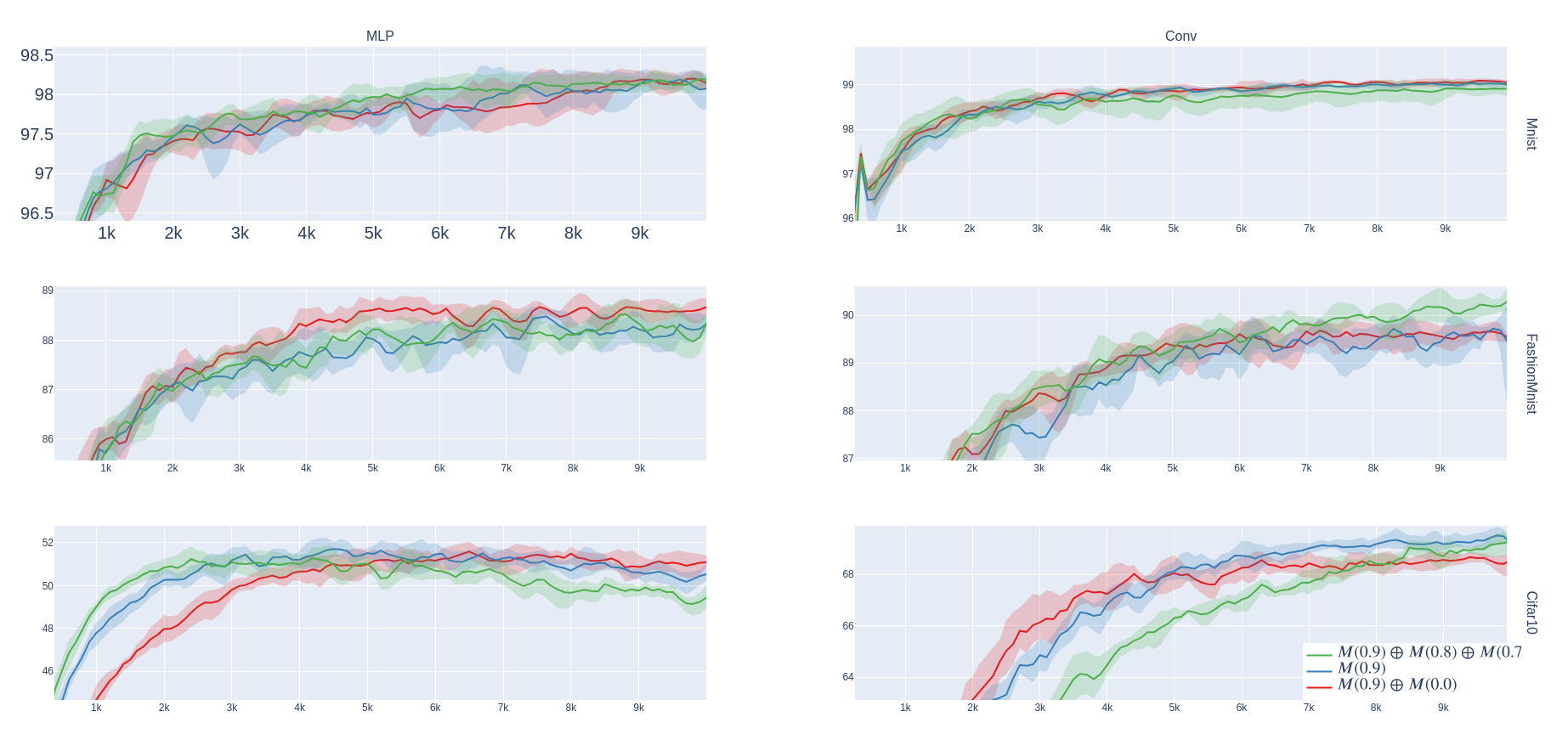}
\caption{Test accuracy graph of some RLLC optimizer, with min-max interval, trained from $3$ random seed initialization. The accuracy does not vary a lot, suggesting, that RLLC is robust.}
\label{fig:error_bar}
\end{figure}